\documentclass{article}
\usepackage{ijcai22}
\usepackage{times}
\usepackage{soul}
\usepackage{url}
\usepackage[hidelinks]{hyperref}
\usepackage[utf8]{inputenc}
\usepackage[small]{caption}
\usepackage{graphicx}
\usepackage{amsmath}
\usepackage{amsthm}
\usepackage{booktabs}
\usepackage{algorithm}
\usepackage{algorithmic}
\usepackage{subfig}
\usepackage{todonotes}
\newtheorem{definition}{Definition}
\usepackage{amssymb}

\newtheorem{theorem}{Theorem}
\newcommand{\R}{\mathbb{R}}

\DeclareMathOperator*{\argmin}{arg\,min}

\title{FedAUXfdp: Differentially Private One-Shot Federated Distillation}

\author{
Haley Hoech
\and
Roman Rischke\and
Karsten M{\"u}ller\And
Wojciech Samek\\
\affiliations
Department of Artificial Intelligence, Fraunhofer Heinrich Hertz Institute\\
\emails
$\{$haley.hoech,
roman.rischke, 
karsten.mueller,
wojciech.samek$\}$@hhi.fraunhofer.de
}

\newcommand\blankfootnote[1]{%
  \begingroup
  \renewcommand\thefootnote{}\footnote{#1}%
  \addtocounter{footnote}{-1}%
  \endgroup
}

\begin{document}

\maketitle

\begin{abstract}
    Federated learning suffers in the case of ``non-iid'' local datasets, i.e., when the distributions of the clients' data are heterogeneous. One promising approach to this challenge is the recently proposed method FedAUX, an augmentation of federated distillation with robust results on even highly heterogeneous client data. FedAUX is a partially $(\epsilon, \delta)$-differentially private method, insofar as the clients' private data is protected in only part of the training it takes part in. This work contributes a \textit{fully differentially private} modification, termed FedAUX\textit{fdp}. We further contribute an upper bound on the $l_2$-sensitivity of regularized multinomial logistic regression. In experiments with deep networks on large-scale image datasets, FedAUXfdp with strong differential privacy guarantees performs significantly better than other equally privatized SOTA baselines on non-iid client data in just a single communication round. Full privatization of the modified method results in a negligible reduction in accuracy at all levels of data heterogeneity.
\end{abstract}

\section{Introduction}
Federated learning (FL)\blankfootnote{\textit{International Workshop on Trustworthy Federated Learning in conjunction with IJCAI 2022 (FL-IJCAI'22), Vienna, Austria.}} is a form of decentralized machine learning, in which a global model is formed by an orchestration server aggregating the outcome of training on a number of local client models without any sharing of their private training data~\cite{mcmahan2017communication}. Interest in federated learning has increased recently for its privacy and communication-efficiency advantages over centralized learning on mobile and edge devices~\cite{li2019federated,sattler2020clustered}. A classical mechanism for model aggregation in FL is federated averaging (FedAVG), where the locally trained models are weighted proportionally to the size of the local dataset. In each communication round of federated averaging, weight updates of the clients' local models are sent to the orchestration server, averaged by the server, and the average is sent back to the federation of clients to initialize the next round of training~\cite{mcmahan2017communication}. 

Federated ensemble distillation (FedD), an often even more communication-efficient and accurate alternative to FedAVG, uses knowledge distillation to transfer knowledge from clients to server~\cite{itahara2020distill,lin2020ensemble,chen2020feddistill,sattler2020communication}. In FedD, clients and server share a public dataset auxiliary to the clients' private data. The clients communicate the output of their privately trained models on the public distillation dataset to the server, which uses the average of these outputs as supervision for the distillation data in training the global model. In comparison to federated averaging, federated ensemble distillation offers additional privacy, as direct white box attacks are not possible for example, and allows combining different model architectures, making it appealing in an Internet-of-Things ecosystem~\cite{li2020covergence,chang2019cronus,li2021fedh2l}. 

FedAUX is an augmentation of federated distillation, which derives its success from taking full advantage of the AUXiliary data. FedAUX uses this auxiliary data for model pretraining and relevance weighting. To perform the weighting, the clients' output on each data point of the distillation dataset is individually weighted by a measure of similarity between that distillation datapoint and the client's local data, called a `certainty score'. Weighting the outputs by the scores prioritizes votes from clients whose local data is more similar to the auxiliary/distillation data. 

A major challenge of federated learning is performance when the distributions of the clients' data are heterogeneous, i.e. performance on ``non-iid'' data, as is often the circumstance in real-world applications of FL~\cite{kairouz2020advances}. FedAUX overcomes that challenge, performing remarkably more efficiently on non-iid data than other state-of-the-art federated learning methods, federated averaging, federated proximal learning, Bayesian federated learning, and federated ensemble distillation. For example on MobilenetV2, FedAUX achieves 64.8\% server accuracy, while even the second-best method only achieves 46.7\%~\cite{sattler2021fedaux}.

Despite its privacy benefits, federated distillation still presents a privacy risk to clients participating~\cite{papernot2017semisupervized}. Data-level differential privacy protects the clients' data by limiting the impact of any individual datapoint on the model and quantifies the privacy loss associated to participating in training with parameters $(\epsilon, \delta)$. Both governments and private institutions are increasingly interested in securing their data using differential privacy. 

Each client in the FedAUX method trains two models on their local dataset. \cite{sattler2021fedaux} privatize only the scoring model, leaving the classification model exposed. The clients' data is accordingly only protected with differential privacy in part of the training it participates in. In this work, we add a local, data-level $(\epsilon, \delta)$-differentially private mechanism for this second model and appropriately modify the FedAUX method to apply said mechanism. We thereby contribute a `fully' privatized version of FedAUX. Full differential privacy here is used as a way of describing our contribution of privacy over the original FedAUX method. We additionally give an upper bound on the $l_2$-sensitivity of regularized multinomial logistic regression. See Section \ref{sec:fdp_priv} for background on differential privacy.

In results with deep neural networks on large scale image datasets at an ($\epsilon = 0.6, \delta = 2*10^{-5}$) level of differential privacy we compare fully differentially private FedAUXfdp with two privatized baselines, federated ensemble distillation and federated averaging in a single communication round. FedAUXfdp outperforms these baselines dramatically on the heterogeneous client data. Our method modifications achieve better results than FedAUX in a single communication round and we see a negligible reduction in accuracy of applying this strong amount of differential privacy to the modified FedAUX method. 

In Section~\ref{sec:fedaux} we outline the original FedAUX, in Section~\ref{sec:exten} we explain our modification, including our privacy mechanism as well as background on differential privacy, and in Section~\ref{sec:experiment} we detail the experimental set-up and highlight important results. 

\section{Related Work}
Our method modifies~\cite{sattler2021fedaux}, who contributed a semi-differentially private FedAUX method. For a discussion of works related to the non-privacy aspects of FedAUX, we refer to their paper.

Cynthia Dwork introduced differential privacy~\cite{dwork2014algorithmic} and~\cite{kasiviswanathan2008dp} local differential privacy. Differential privacy bounds were greatly improved with the introduction of the moments accountant in~\cite{abadi2016deep}. 

In addition to quantifying privacy loss, differential privacy protects provably against membership inference attacks~\cite{shokri2017mi,choquette2021label}, in which an adversary can determine if a data point participated in the training of a model. This can pose a privacy threat, for example, if participation in model training could imply a client has a particular disease or other risk factor. Alternatives for privatization in general include secure multi-party computation or homomorphic encryption, though neither protect against membership inference attacks~\cite{shokri2017mi}. 
Others have combined local differential privacy and federated learning, notably~\cite{geyer2018dpfl,mcmahan2018dp}. While~\cite{sun2020federated} combined federated model distillation with differential privacy, they only attain robust results on non-iid data when the client and distillation data contains the same classes.

Our one-shot federated distillation approach uses model distillation to transfer the learning outcome in form of softlabels for a public distillation dataset. \cite{ZhouEtAl2021} use dataset distillation \cite{WangEtAl2018} to design a communication-efficient and privacy-preserving one-shot FL mechanism.

\section{FedAUX}
\label{sec:fedaux}

\subsection{Method}

In FedAUX, there are two actors, the clients and the orchestration server. Each client, $i=1,\ldots,n$, has its own private, local, labeled dataset $D_i$. Auxiliary to the client data, is a public, unlabeled dataset $D_{aux}$. The auxiliary data is further split into the negative data $D^-$, used in training the certainty score models, and the distillation data $D_{distill}$, used for knowledge distillation. 

There are three types of models, the clients' scoring models, the clients' classification models, and the server's global model, which can all be decomposed into a feature extractor $h$ and linear or logistic regression classification head. Whether the full model or just the classification head is trained varies by model and we outline this next. In FedAUX four kinds of training are conducted (see Figure~\ref{fig:method}):

 \begin{figure}[ht]
    \centering
    \includegraphics[width=\linewidth]{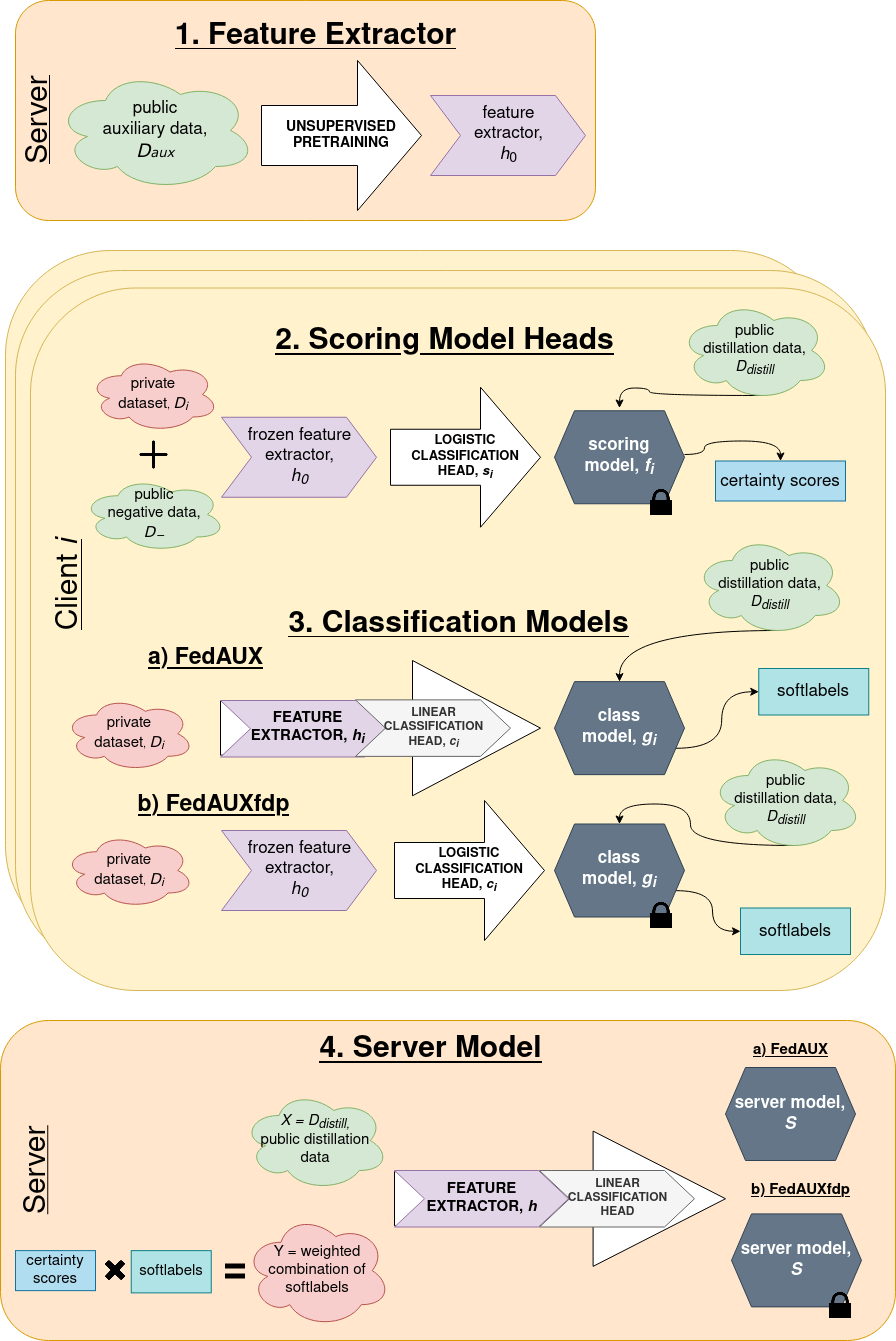}
    \caption{Overview of FedAUX and FedAUXfdp training. Models secured with differential privacy indicated with black locks.}
   \label{fig:method}
\end{figure}

\begin{enumerate}
    \item \textbf{Feature extractor.} Unsupervised pretraining with the public auxiliary data $D_{aux}$ on the server to obtain the feature extractor, $h_0$, which is sent to the clients and initialized in all their models as well as the server's.
    \item \textbf{Scoring model heads.}  Supervised training of the scoring model classification heads $s_i$ of all clients, in combination with the frozen feature extractor $h_0$ to generate scoring models $f_i = s_i \circ h_0$. Each training is a binary logistic regression on the extracted features of their private local data and the public negative data $h_0(D_i \cup D^-)$. 
    \item \textbf{Classification models.} Supervised training of the clients' full classification models $g_i = c_i \circ h_i$, consisting of a feature extractor $h_i$ (initialized with $h_0$ from the pretraining) and linear classification head $c_i$, on their local datasets $D_i$. 
    \item \textbf{Server model.} Supervised training of the server's full model $S$, consisting of a feature extractor $h$ (initialized with $h_0$ from the pretraining) and linear classification head. The server calculates an initial weight update of the clients' average class model weight updates from their training round. For the server's training, the input data $X$ is the unlabeled $D_{distill}$ and the supervision $Y$ a $(|D_{distill}| \times n_{classes})$-dimensional matrix of the softlabel output of the class model $g_i(D_{distill})$, weighted by a certainty score for each distillation datapoint. The certainty scores are the output of the $(\epsilon,\delta)$-differentially privatized scoring model on the distillation data $f_i(D_{distill})$, measures of similarity between each distillation data point and the client's local data. Each entry in $Y$ is:

\begin{equation}
    \frac{\sum_i f_i(x) \cdot g_i(x)}{\sum_i f_i(x)}, \text{for } x \in D_{distill}.
\end{equation}
\end{enumerate}

\subsection{Privacy}

Participating in the training of the scoring classification heads and classification models presents a privacy risk to the private data of the clients. In FedAUX, the scoring heads are sanitized using an $(\epsilon, \delta)$-differentially private sanitization mechanism. FedAUX's mechanism for privatizing the scoring model is based on freezing the feature extractor and using a logistic classification head. As the feature extractor was trained on public data, only sanitizing this head is required to yield a differentially private model. Further, using the L-BFGS optimizer in sci-kit learn's logistic regression guarantees finding optimal weights for the logistic regression heads. In FedAUXfdp we privatize the classification models in a similar fashion. This thereby makes the server models learned in FedAUXfdp fully differentially private, as discussed in Section \ref{sec:exten}, with the specific privacy mechanism outlined in Section~\ref{sec:fdp_priv}. 

\section{FedAUXfdp}
\label{sec:exten}
In the fully differentially private version of FedAUX, we adapt the training of the classification and server models as follows. Rather than training the full client models, we freeze the feature extractors and train only the classification heads using a multinomial logistic regression on extracted features of the client's local dataset $D_i$. As communicating model updates to the server poses a privacy threat, we no longer initialize the server with the averaged weight update of the clients. Accordingly, step three in the process is changed as follows:

\begin{enumerate}
    \setcounter{enumi}{2}
    \item \textbf{Classification model.} Supervised training of the classification model heads $c_i$ of the clients, combined with the frozen feature extractor $h_0$, to generate class models $g_i = c_i \circ h_0$. Each is a multinomial logistic regression on the extracted features of their private local data $h_0(D_i)$. See Figure \ref{fig:method}.
\end{enumerate}

As with the scoring models in the original FedAUX, freezing the feature extractors, which have been trained on public data, allows us to make the models differentially private by simply sanitizing the classification heads. Again, we opt for logistic classification heads because the L-BFGS optimizer in sci-kit learn's logistic regression guarantees convergence to globally optimal weights of the logistic regression. 

 We formulate the training of these classifiers as regularized empirical risk minimization problems.

\subsection{Regularized Empirical Risk Minimization}
\label{sec:erm}

Let $\boldsymbol{\beta} := (\boldsymbol{\beta}^T_1, \ldots, \boldsymbol{\beta}^T_C)^T \in \R^{C(p+1)}$ with $\boldsymbol{\beta}_k := (\beta_{k,0}, \ldots, \beta_{k,p})^T \in \R^{p+1}$ be the vector of trainable parameters of  the regularized multinomial logistic regression problem with $C$ classes
\begin{equation}\label{multilogreg}
  \min_{\boldsymbol{\beta}} J(\boldsymbol{\beta},h,D) = \frac{1}{|D|} \sum_{i=1}^{|D|} -\log(p_{y_i}(h(\textbf{x}_i))) + \frac{\lambda}{2} \| \boldsymbol{\beta}\|_2^2
\end{equation} 
with softmax function \[p_{y_i}(\boldsymbol{\beta}, h(\textbf{x}_i)) = \frac{\exp(\boldsymbol{\beta}_{y_i}^T h(\textbf{x}_i))}{\sum_{k=1}^C \exp(\boldsymbol{\beta}_k^T h(\textbf{x}_i))} \]
for a labeled data point $(\textbf{x}_i, y_i)$ from a dataset $D$. 

Thereby, $h(\textbf{x}_i) \in \R^{p+1}$ is an extracted feature vector with the first coordinate being a constant for the bias term $\beta_{k,0}$, and $y_i\in \{1,\ldots,C\}$ the corresponding class label. We assume w.l.o.g. that 
\begin{equation}\label{eq:assumption}
  \|h(\textbf{x})\|_2 \leq 1.
\end{equation}

To fulfill this assumption, we normalize the input features for the logistic regression problem as follows
\begin{equation}
    \Tilde{h}(\textbf{x}) := h(\textbf{x})\Big(\operatorname*{max}_{\textbf{x} \in D} \Vert h(\textbf{x})\Vert_2 \Big)^{-1}.
\end{equation}

\subsection{Privacy}
\label{sec:fdp_priv}
We privatize the classification models using $(\epsilon, \delta)$-differential privacy. Informally, differential privacy anonymizes the client data in this context, insofar as with very high likelihood the results of the model would be very similar regardless whether or not a particular data point participates in training~\cite{dwork2014algorithmic}.
\subsubsection{Definitions}

\begin{definition}
A randomized mechanism $\mathcal{M}: \mathcal{D} \rightarrow \mathcal{R}$ satisfies $(\epsilon, \delta)$-differential privacy, if for any two adjacent inputs $D_1$ and $D_2$ that only differ in one element and for any subset of outputs $S \subseteq \mathcal{R},$  
$$P[D_1 \in S] \leq \exp(\epsilon)P[\mathcal{M}(D_2)\in S] + \delta.$$
\end{definition} 

We use the Gaussian mechanism, in which a specific amount of Gaussian noise is added relative to the $l^2$-\textit{sensitivity}~\cite{dwork2014algorithmic} and according to pre-selected $\epsilon$ and $\delta$ values.

\begin{definition}
For $\epsilon \in (0,1)$, $c^2 > 2\ln(1.25/\delta)$, the Gaussian Mechanism with parameter $\sigma \geq c\Delta_2(\mathcal{M})/\epsilon$ is $(\epsilon,\delta)$-differentially private.
\end{definition}
\begin{definition}
For two datasets, $D_1, D_2$ differing in one datapoint, the $l_2$-sensitivity is  $$\Delta(\mathcal{M}) = \operatorname*{max}_{D_1,D_2 \in \mathcal{D}}\Vert \mathcal{M}(D_1) - \mathcal{M}(D_2)\Vert_2 $$
\end{definition}

\subsubsection{Sensitivity of the Classification Models}

We contribute the following theorem for the $l_2$-sensitivity of regularized multinomial logistic regression \eqref{multilogreg}, which generalizes a corollary from~\cite{chaudhuri2011differentially}.  

\begin{theorem}
\label{theorem1}
The $l_2$-sensitivity of regularized multinomial logistic regression, as defined in \eqref{multilogreg}, is at most $\frac{2\sqrt{C}}{\lambda |D|}$. 
\end{theorem}
\begin{proof}
W.l.o.g. we set $h(\textbf{x})=\textbf{x}$ in this proof and omit the argument $h$ in the definition of $J$ for ease of exposition. Let $D=\{(\textbf{x}_1, y_1), \ldots, (\textbf{x}_N, y_N)\}$ and $D'= (D\setminus\{(\textbf{x}_N, y_N)\})\cup \{(\textbf{x}'_N, y'_N)\}$. That is, $D$ and $D'$ differ in exactly one data point. Furthermore, let
\begin{eqnarray}
  \boldsymbol{\beta}_1^* &=& \argmin_{\boldsymbol{\beta}} J(\boldsymbol{\beta}, D) \\ \boldsymbol{\beta}_2^* &=& \argmin_{\boldsymbol{\beta}} J(\boldsymbol{\beta}, D').
\end{eqnarray}
The goal is to show that $\|\boldsymbol{\beta}_1^* - \boldsymbol{\beta}_2^*\|_2 \leq \frac{2\sqrt{C}}{\lambda N}$. We define 
\begin{eqnarray}
  d(\boldsymbol{\beta}) &:=& J(\boldsymbol{\beta}, D') - J(\boldsymbol{\beta}, D) \\
  &=& \frac{1}{N}\left( l(\boldsymbol{\beta}, \textbf{x}'_N) - l(\boldsymbol{\beta}, \textbf{x}_N)\right),\notag
\end{eqnarray}
with the log-softmax loss function 
\begin{eqnarray}
  l(\boldsymbol{\beta}, \textbf{x}) := -\log(p_y(\boldsymbol{\beta}, \textbf{x}))
\end{eqnarray}
for an arbitrary data point $(\textbf{x}, y)$.

With 
\begin{eqnarray}
  \nabla_{\boldsymbol{\beta}} \ l(\boldsymbol{\beta}, \textbf{x}) &=& -\frac{\nabla_{\boldsymbol{\beta}} \ p_y(\boldsymbol{\beta}, \textbf{x})}{p_y(\boldsymbol{\beta}, \textbf{x})} 
\end{eqnarray}
we obtain 
\begin{eqnarray}
  \frac{\partial \ p_k(\boldsymbol{\beta}, \textbf{x})}{\partial \ \boldsymbol{\beta}_k} &=& \frac{\exp(\boldsymbol{\beta}_k^T \textbf{x}) \cdot \sum_{j\neq k} \exp(\boldsymbol{\beta}_j^T \textbf{x})}{(\sum_{j} \exp(\boldsymbol{\beta}_j^T \textbf{x}))^2} \textbf{x} \\
  \frac{\partial \ p_k(\boldsymbol{\beta}, \textbf{x})}{\partial \ \boldsymbol{\beta}_{\ell\neq k}} &=& - \frac{\exp(\boldsymbol{\beta}_k^T \textbf{x})\cdot \exp(\boldsymbol{\beta}_\ell^T \textbf{x})}{(\sum_{j} \exp(\boldsymbol{\beta}_j^T \textbf{x}))^2} \textbf{x} \\
  \frac{\partial \ l(\boldsymbol{\beta}, \textbf{x})}{\partial \ \boldsymbol{\beta}_{k=y}} &=& \frac{\sum_{j\neq k} \exp(\boldsymbol{\beta}_j^T \textbf{x})}{\sum_{j} \exp(\boldsymbol{\beta}_j^T \textbf{x})}\textbf{x}\label{eq:grad1}\\
  \frac{\partial \ l(\boldsymbol{\beta}, \textbf{x})}{\partial \ \boldsymbol{\beta}_{\ell \neq y}} &=& -\frac{\exp(\boldsymbol{\beta}_\ell^T \textbf{x})}{\sum_{j} \exp(\boldsymbol{\beta}_j^T \textbf{x})}\textbf{x}.\label{eq:grad2}
\end{eqnarray}
Note, that the factors on the rhs of \eqref{eq:grad1} and \eqref{eq:grad2} have absolute values of at most 1. Hence, we can bound
\begin{eqnarray}
  \|\nabla_{\boldsymbol{\beta}} \ d(\boldsymbol{\beta})\|_2 &=& \frac{1}{N}\|\nabla_{\boldsymbol{\beta}} \ l(\boldsymbol{\beta}, \textbf{x}'_N) - \nabla_{\boldsymbol{\beta}} \ l(\boldsymbol{\beta}, \textbf{x}_N )\|_2 \\
  &\leq & \frac{1}{N}\left(\|\nabla_{\boldsymbol{\beta}} \ l(\boldsymbol{\beta}, \textbf{x}'_N)\|_2 + \|\nabla_{\boldsymbol{\beta}} \ l(\boldsymbol{\beta}, \textbf{x}_N )\|_2 \right)\notag\\
  &\leq & \frac{1}{N} \left(\sqrt{C}\|\textbf{x}'_N\|_2 + \sqrt{C}\|\textbf{x}_N\|_2 \right)\notag\\
  &\leq & \frac{2\sqrt{C}}{N},\notag
\end{eqnarray}
where the last inequality follows from assumption \eqref{eq:assumption} that $\|\textbf{x}\|_2 \leq 1$.

We observe that due to the convexity of $l(\boldsymbol{\beta}, \textbf{x})$ in $\boldsymbol{\beta}$ and the 1-strong convexity of the $l_2$-regularization term in \eqref{multilogreg}, $J(\boldsymbol{\beta},D)$ is $\lambda$-strongly convex. Hence, we obtain by Shalev-Shwartz inequality \cite{shalev2007online}
\begin{eqnarray}
  \left(\nabla_{\boldsymbol{\beta}} \ J(\boldsymbol{\beta_1^*},D) - \nabla_{\boldsymbol{\beta}} \ J(\boldsymbol{\beta_2^*},D)\right)^T \left( \boldsymbol{\beta_1^*} - \boldsymbol{\beta_2^*}\right) \geq \notag \\ \lambda \| \boldsymbol{\beta_1^*} - \boldsymbol{\beta_2^*} \|_2^2. 
\end{eqnarray}
Moreover, by construction of $d(\boldsymbol{\beta}),$
\begin{equation}
  J(\boldsymbol{\beta_2^*},D) + d(\boldsymbol{\beta_2^*}) = J(\boldsymbol{\beta_2^*},D').
\end{equation}
By optimality of $\boldsymbol{\beta_1^*}$ and $\boldsymbol{\beta_2^*}$, it holds
\begin{eqnarray}
  \textbf{0} = \nabla_{\boldsymbol{\beta}} \ J(\boldsymbol{\beta_1^*},D) &=& \nabla_{\boldsymbol{\beta}} \ J(\boldsymbol{\beta_2^*},D') \\ &=& \nabla_{\boldsymbol{\beta}} \ J(\boldsymbol{\beta_2^*},D) + \nabla_{\boldsymbol{\beta}} \ d(\boldsymbol{\beta_2^*}). \notag
\end{eqnarray}
Applying the Cauchy-Schwartz inequality finally leads to
\begin{eqnarray}
 &&\|\boldsymbol{\beta_1^*} - \boldsymbol{\beta_2^*} \|_2 \cdot  \| \nabla_{\boldsymbol{\beta}} \ d(\boldsymbol{\beta_2^*})\|_2 \geq  \left( \boldsymbol{\beta_1^*} - \boldsymbol{\beta_2^*}\right)^T \nabla_{\boldsymbol{\beta}} \ d(\boldsymbol{\beta_2^*}) \notag\\
  &=& \left( \boldsymbol{\beta_1^*} - \boldsymbol{\beta_2^*}\right)^T \left( \nabla_{\boldsymbol{\beta}} \ J(\boldsymbol{\beta_1^*},D) - \nabla_{\boldsymbol{\beta}} \ J(\boldsymbol{\beta_2^*},D)\right) \notag\\
  &\geq &  \lambda \|\boldsymbol{\beta_1^*} - \boldsymbol{\beta_2^*} \|_2^2,
\end{eqnarray}
which concludes the proof, since
\begin{equation}
  \|\boldsymbol{\beta_1^*} - \boldsymbol{\beta_2^*} \|_2 \leq \frac{\| \nabla_{\boldsymbol{\beta}} \ d(\boldsymbol{\beta_2^*})\|_2}{\lambda} \leq \frac{2\sqrt{C}}{\lambda N}.
\end{equation} 
\end{proof}

We remark that in the binary case ($C=2$) one regression head parameterized by $\boldsymbol{\beta} \in \R^{(p+1)}$ suffices, resulting in an $l_2$-sensitivity of at most $\frac{2}{\lambda |D|}$.

\subsubsection{Private Mechanism}

Using Theorem \ref{theorem1} and the Gaussian mechanism, we get our $(\epsilon,\delta)$-differentially private mechanism for sanitizing the multinomial classification models as follows:

$$\mathcal{M}_{priv}(D) = \mathcal{M}(D) + \mathcal{N}(0,I\sigma^2),\text{ where }$$ 
$$\sigma^2 = \frac{8Cln(1.25\delta^{-1})}{\epsilon^2\lambda^2|D|^2}$$

This leads to the overall training procedure for the classification models described in Algorithm~\ref{alg}.

\begin{algorithm}[t]
\caption{Classification model training and privatization}
\begin{algorithmic}\label{alg}
\FOR{each client } 
    \STATE $\boldsymbol{\beta}^* \rightarrow \operatorname*{argmin}_{\boldsymbol{\beta}} J(\boldsymbol{\beta},h,D) $
    \STATE $\sigma^2 \rightarrow \frac{8C\ln(1.25\delta^{-1})}{\epsilon^2\lambda^2(|D|)^2}$
    \STATE $\boldsymbol{\beta}^* \rightarrow \boldsymbol{\beta}^* +  \mathcal{N}(0,I\sigma^2)$
\ENDFOR
\end{algorithmic}
\end{algorithm}

\subsection{Cumulative Privacy Loss}
\label{sec:fdp}
By the composability and post-processing properties of differentially private mechanisms~\cite{dwork2014algorithmic}, the cumulative privacy loss for an individual client's dataset in training of the server's model is equal to the sum of the loss of the scoring and classification models. The server model is $(\epsilon,\delta)$-differentially private, where

$$\epsilon = \epsilon_{scores} + \epsilon_{classes}$$
$$\delta = \delta_{scores} + \delta_{classes}$$

\section{Experiments}
\label{sec:experiment}

We ran experiments on large-scale convolutional, ShuffleNet-~\cite{zhang2018ShuffleNet}, MobileNet-~\cite{sandler2018mobilenetv2},  and ResNet-style~\cite{he2016deep} networks, using CIFAR-10 as local client data and both STL-10 and CIFAR-100 as auxiliary data. Of the auxiliary data, 80\% is used for distillation and 20\% for unsupervised pretraining. The pretraining is done by contrastive representation learning using the Adam optimizer with a learning rate of $10^{-3}.$  

The number of clients is $n=20$ and there is full participation in one round of communication. The training data is split among the clients using a Dirichlet distribution with parameter $\alpha$ as done first in~\cite{hsu2019measuring} and later in~\cite{lin2020ensemble,chen2020feddistill}. With the lowest $\alpha = 0.01$, clients see almost entirely one class of images. With the highest $\alpha=10.24$, each client sees a substantial number of images from every class. See Table \ref{tab:iidness}. We follow \cite{sattler2021fedaux} in their selection of highlighted Dirichlet parameters $\alpha$, who chose $\alpha = 2^n*10^{-2},$ for $n \in \{0, 2, 4, 10\}$.

\begin{table}[h]
    \centering
    \begin{tabular}{lrrrr}
        \toprule
         Class & $\alpha$ = 0.01 & $\alpha$ = 0.04 & $\alpha$  = 0.16 & $\alpha$ = 10.24 \\
         \midrule
         First & 94.5\% & 75.3\% & 56.8\% & 15.1\%\\
         Second & 5.2\% & 16.6\% & 22.3\% & 13.6\%\\
         Third & 0.3\% & 5.6\% & 10.1\%& 12.0\%\\
         \bottomrule
    \end{tabular}
    \caption{Ranked percentage of data coming from the three largest classes for each level of data heterogeneity}
    \label{tab:iidness}
\end{table}

We find the optimal weights of the class model logistic regressions using scikit-learn's LogisticRegression with the L-BFGS~\cite{liu1989limited} optimizer. For baselines, we chose Federated Ensemble Distillation (FedD) and Federated Averaging (FedAVG), which we pretrain (+P) in the same fashion as FedAUXfdp. We also compare FedAUXfdp to FedAUX, but with a frozen feature extractor (+F) for consistency. In FedAUX+F, the clients' local models (linear classification heads) are trained for 40 local epochs. For FedAUX+F, FedAUXfdp, and FedD+P, the full sever model is trained for 10 distillation epochs using the Adam optimizer with a learning rate of $5\cdot 10^{-5}$ and a batch size of 128. For FedAVG+P, the average of the weights of the clients' logistic regressions is used as a classification head on top of the frozen feature extractor on the server. 

For privacy, we chose $(\epsilon=0.1, \delta=10^{-5})$ for the scores and unless otherwise mentioned $(\epsilon=0.5, \delta=10^{-5})$ for the classes. We choose regularization parameter $\lambda=0.01$ for both the certainty score and class models unless otherwise mentioned. 
\begin{table*}[ht]
    \centering
    \scalebox{0.97}{
    \begin{tabular}{lrrrrcrrrr}
        \toprule 
         & \multicolumn{4}{c}{ShuffleNet} & & \multicolumn{4}{c}{MobileNetv2} \\
        \cline{2-5} 
        \cline{7-10} \\
        Method &            $\alpha=0.01$ &   $\alpha=0.04$     &    $\alpha=0.16$ &           $\alpha=10.24$ &      &            $\alpha=0.01$ &   $\alpha=0.04$     &    $\alpha=0.16$ &           $\alpha=10.24$  \\  
        \midrule
        FedAVG+P & 46.0$\pm$ 0.4 & 56.7$\pm$ 6.6 & 67.5$\pm$ 3.5 & 74.1 $\pm$ 1.4 & & 47.2$\pm$ 2.6 & 54.2$\pm$ 5.5 & 65.6$\pm$ 0.9 & 72.0 $\pm$ 0.6  \\
        FedD+P & 41.8 $\pm$ 4.4 & 54.7 $\pm$ 5.0 & 68.8 $\pm$ 2.1 & 72.3 $\pm$ 1.6 & & 43.7 $\pm$ 1.8 & 52.2 $\pm$  4.6 & 67.0 $\pm$  1.7 & 70.8  $\pm$  0.2\\
        FedAUXfdp  & \textbf{75.2 $\pm$ 1.1} & \textbf{74.6 $\pm$ 1.1} & 72.3 $\pm$ 0.6 & 71.7 $\pm$ 1.3 & & \textbf{72.8 $\pm$ 0.4 } & \textbf{72.0 $\pm$ 1.2} & 70.8 $\pm$ 0.2 & 69.4 $\pm$ 0.8  \\
        \bottomrule
    \end{tabular}
    }
    \caption{Server model inference accuracy of \textbf{FedAUXfdp as compared to FL baselines} with pretraining. All methods cumulative privacy loss ($\epsilon=0.6, \delta=2e-05$).}
    \label{tab:baselines}
\end{table*}
As shown in Table~\ref{tab:baselines}, on both ShuffleNet and MobileNetv2 architectures FedAUXfdp significantly outperforms baselines in the most heterogeneous settings ($\alpha= 0.01, 0.04).$ While the baselines undergo a steady reduction in accuracy as client data heterogeneity increases, FedAUXfdp is even improving. As data heterogeneity increases fewer classes per client result in the addition of less noise, see Theorem~\ref{theorem1}.

\begin{table*}[ht]
    \centering
    \scalebox{0.86}{
        \begin{tabular}{llrrrrcrrrr}
            \toprule
             & & \multicolumn{4}{c}{ShuffleNet} & & \multicolumn{4}{c}{MobileNetv2} \\
            \cline{3-6}
            \cline{8-11}\\
            Method &    Class DP &        $\alpha=0.01$ &   $\alpha=0.04$     &    $\alpha=0.16$ &           $\alpha=10.24$ &      &            $\alpha=0.01$ &   $\alpha=0.04$     &    $\alpha=0.16$ &           $\alpha=10.24$ \\
            \midrule
            FedAUX+F & None & 64.8 $\pm$ 1.1 & 64.9 $\pm$ 0.5 & 67.7 $\pm$ 0.8 & 73.4 $\pm$ 0.1 & & 60.1 $\pm$ 1.2 & 61.2 $\pm$ 1.8 & 63.7 $\pm$ 0.8 & 67.5 $\pm$ 0.0 \\
            \midrule
            FedAUXfdp & None & 76.1 $\pm$ 0.3 & 75.6 $\pm$ 0.4 & 75.2 $\pm$ 0.5 & 75.4 $\pm$ 0.1 & & 73.0 $\pm$ 0.5 & 73.3 $\pm$ 0.6 & 73.2 $\pm$ 0.2 & 73.0 $\pm$  0.1 \\
            FedAUXfdp & (1.0, 1e-05) & 75.7 $\pm$ 0.7 & 75.1 $\pm$  0.7 & 74.6  $\pm$ 0.5 & 74.9  $\pm$ 0.2 & & 73.0  $\pm$ 0.4 & 72.7  $\pm$ 1.0 & 72.7  $\pm$ 0.3 & 72.4  $\pm$ 0.0  \\
            FedAUXfdp & (0.5, 1e-05)  & 75.2 $\pm$ 1.1 & 74.6 $\pm$ 1.1 & 72.3 $\pm$ 0.6 & 71.7 $\pm$ 1.3 & & 72.8 $\pm$ 0.4 & 72.0 $\pm$ 1.2 & 70.8 $\pm$ 0.2 & 69.4 $\pm$ 0.8  \\
            FedAUXfdp & (0.1, 1e-05) & 60.8 $\pm$ 2.4 & 59.4 $\pm$ 5.8 & 33.9 $\pm$ 5.4 & 34.6 $\pm$ 3.0 & & 66.4 $\pm$ 3.3 & 53.1 $\pm$ 12.9 & 38.9 $\pm$ 4.4 & 34.9 $\pm$ 3.3 \\
            FedAUXfdp & (0.01, 1e-05) & 36.3 $\pm$ 5.1 & 39.8 $\pm$ 7.5 & 12.6 $\pm$ 5.1 & 11.7 $\pm$ 3.5 & & 44.4 $\pm$ 6.8 & 28.7 $\pm$ 5.1 & 16.6 $\pm$ 5.8 & 11.5 $\pm$ 0.8\\
            \bottomrule
        \end{tabular}
        }
    \caption{FedAUXfdp server model inference accuracy at \textbf{various levels of class differential privacy} $(\epsilon,\delta)$ and \textbf{comparison to FedAUX+F} server model inference accuracy. Scoring model privacy for all methods ($\epsilon=0.1, \delta=1e-05$). }
    \label{tab:fedauxfdp}
\end{table*}

Table~\ref{tab:fedauxfdp} shows the impact on the server model accuracy from the method modifications and from different levels of privacy in FedAUXfdp. The method modifications (FedAUXfdp without the class differential privacy) are an all-around improvement in accuracy over FedAUX+F, especially on non-iid client data. The logistic classification heads outperform the linear ones in a single communication round. For FedAUXfdp with no class differential privacy the results are nearly constant as opposed to FedAUX+F, where one sees the usual improvement as iid-ness increases. 

Privatizing FedAUXfdp at additional epsilon-delta values of $(0.5, 10^{-5})$ results in nearly no reduction in accuracy over FedAUXfdp with no class model privacy. Only at $\epsilon=0.1$ we see a drop in accuracy. With equal regularization, the additional differential privacy impacts the models trained on the non-iid data distributions less than those trained on homogeneous data, again due to the class size term $C$ in the $l_2$-sensitivity from Theorem \ref{theorem1}.

\begin{figure*}[!ht]%
    \centering
    {{\includegraphics[width=4.35cm]{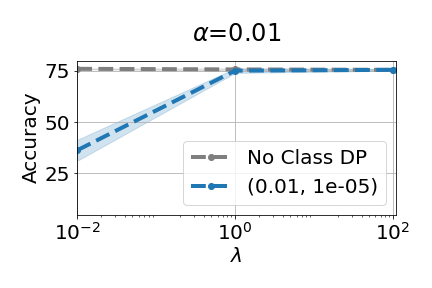} }}%
    {{\includegraphics[width=4.35cm]{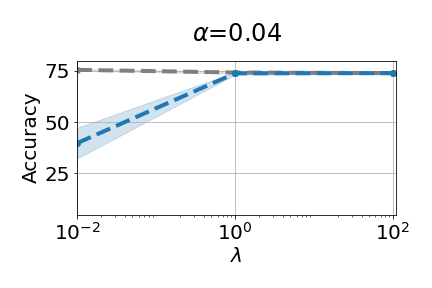} }}%
    {{\includegraphics[width=4.35cm]{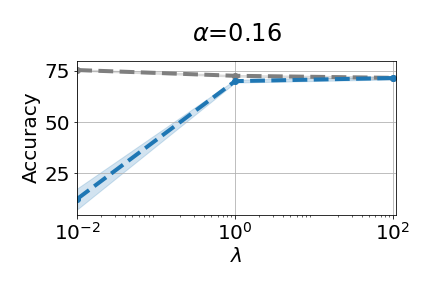} }}%
    {{\includegraphics[width=4.35cm]{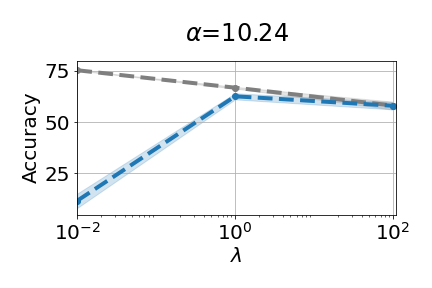} }}%
    \caption{\textbf{Accuracy vs. regularization}. Server model inference accuracy of FedAUXfdp on ShuffleNet with and without $(\epsilon = 0.01, \delta = 1e-05)$ class model differential privacy at various levels of class model regularization $\lambda$. }%
    \label{fig:reg}%
\end{figure*}

\begin{table*}[!ht]
    \centering
    \begin{tabular}{lrrrr}
        \toprule
        Distill Data &            $\alpha=0.01$ &   $\alpha=0.04$     &    $\alpha=0.16$ &           $\alpha=10.24$   \\ 
        \midrule
        STL-10 & 77.2 $\pm$ 0.5 & 75.4 $\pm$ 1.0 & 74.7 $\pm$ 0.9  & 74.4 $\pm$ 0.8 \\
        CIFAR-100 & 70.4 $\pm$ 0.7 & 68.9 $\pm$ 1.8 & 67.6 $\pm$ 1.6 & 68.5 $\pm$ 1.9 \\
        \bottomrule
    \end{tabular}
    \caption{FedAUXfdp server model inference accuracy on ResNet8 with distillation data sharing 9/10 classes (STL-10) versus \textbf{entirely different distillation data classes} (CIFAR-100).}
    \label{tab:distill}
\end{table*}

The drop in accuracy of adding class differential privacy can be partially compensated for by increasing the regularization parameter $\lambda$ of the client models' logistic regressions. Regularization reduces model variance and therefore the impact an individual datapoint has on the model. It thus affects the sensitivity of a differentially private mechanism as in the corollary from \cite{chaudhuri2011differentially}, which our sensitivity theorem generalizes. As shown in Figure~\ref{fig:reg}, on the ShuffleNet model architecture, increasing the regularization from $\lambda=0.01$ to $\lambda=1$ nearly eliminates the gap between the accuracy with and without $(\epsilon = 0.01, \delta = 10^{-5})$ class model differential privacy at all levels of data heterogeneity $\alpha$. The additional regularization does, however, reduce the accuracy of the model without the class differential privacy, moreso the more homogeneous the client data.

Table~\ref{tab:distill} shows results on ResNet with both STL-10 and CIFAR-100 as distillation data. STL-10 and CIFAR10 share 9/10 of the same classes, while CIFAR-100 has completely different classes. Even with distillation classes unmatching client classes, we still see robust results.

\section{Conclusion}

In this work, we have modified the FedAUX method, an augmentation of federated distillation, and made it fully differentially private. We have contributed a mechanism that privatizes respectably with little loss in model accuracy, particularly on non-iid client data. We additionally contributed a theorem for the sensitivity of $l_2$ regularized multinomial logistic regression. On large scale image datasets we have examined the impact of different amounts of differential privacy and regularization. Measuring the impact of federated averaging, distillation, and differential privacy on the attackability of the global server model would be an interesting investigation direction.

\bibliographystyle{named}
\bibliography{references.bib}
\end{document}